\documentclass[11pt]{article}
\usepackage[margin=1in]{geometry}
\usepackage[T1]{fontenc}
\usepackage{lmodern}
\usepackage{amsmath,amssymb,amsthm,mathtools,bm}
\usepackage{booktabs,tabularx,array}
\usepackage{algorithm}
\usepackage{placeins}
\usepackage[noend]{algpseudocode}
\usepackage{enumitem}
\usepackage{microtype}
\usepackage{xcolor}
\usepackage{mdframed}
\newmdenv[backgroundcolor=black!3,linecolor=black!65,linewidth=1.2pt,
  topline=false,bottomline=false,rightline=false,
  innerleftmargin=10pt,innerrightmargin=10pt,
  innertopmargin=8pt,innerbottommargin=8pt,
  skipabove=10pt,skipbelow=10pt,nobreak=true]{researchquestion}
\usepackage{graphicx}
\usepackage[hidelinks]{hyperref}
\usepackage[nameinlink,noabbrev]{cleveref}
\crefname{assumption}{assumption}{assumptions}
\Crefname{assumption}{Assumption}{Assumptions}
\crefname{corollary}{corollary}{corollaries}
\Crefname{corollary}{Corollary}{Corollaries}
\crefname{algorithm}{algorithm}{algorithms}
\Crefname{algorithm}{Algorithm}{Algorithms}

\newtheorem{theorem}{Theorem}[section]
\newtheorem{lemma}[theorem]{Lemma}

\newtheorem{corollary}[theorem]{Corollary}
\newtheorem{remark}[theorem]{Remark}
\newtheorem{assumption}[theorem]{Assumption}

\newcommand{\R}{\mathbb R}
\newcommand{\E}{\mathbb E}
\newcommand{\Prb}{\mathbb P}
\newcommand{\KL}{\mathrm{KL}}

\newcommand{\cW}{\mathcal W}
\newcommand{\cS}{\mathcal S}
\newcommand{\Reg}{\operatorname{Reg}}

\newcommand{\ip}[2]{\langle #1,#2\rangle}
\newcommand{\N}{\mathcal N}

\newcommand{\norm}[1]{\left\lVert #1\right\rVert}

\newcommand{\polylog}{\operatorname{polylog}}

\newcommand{\method}{\textup{\textsc{SparseGauss}}}

\allowdisplaybreaks
\newif\ifanonymous
\anonymousfalse
\newcommand{\AuthorNames}{Yanyi Li, Wendao Wu, Haihan Zhang, Qinshi Wang, Chenheng Zhang, Chunyuan Zheng, Cong Fang, Haoxuan Li, Zhouchen Lin}
\newcommand{\AuthorAffiliation}{Peking University}
\newcommand{\Email}[1]{\href{mailto:#1}{\texttt{#1}}}
\newcommand{\AuthorBlock}{%
Yanyi Li$^{1,*}$\quad Wendao Wu$^{1,*}$\quad Haihan Zhang$^{1,*}$\\[0.20em]
Qinshi Wang$^{1,*}$\quad Chenheng Zhang$^{1,*}$\quad Chunyuan Zheng$^{1}$\\[0.20em]
Cong Fang$^{1,\dagger}$\quad Haoxuan Li$^{1,\dagger}$\quad Zhouchen Lin$^{1,\dagger}$\\[0.60em]
{\small $^{1}$\AuthorAffiliation}\\[0.35em]
{\small \Email{colombo.2026@outlook.com}}\\[-0.05em]
{\small $^*$Equal contribution.\quad $^\dagger$Corresponding authors.}}
\newcommand{\ResearchAgentSystem}{Colombo, our laboratory's internal auto-research system}
\title{\textbf{Dimension-Optimal Dynamic Regret in Exp-Concave Online Learning}}
\ifanonymous
\author{Anonymous Authors}
\hypersetup{pdfauthor={Anonymous Authors}}
\else
\author{\AuthorBlock}
\hypersetup{pdfauthor={\AuthorNames}}
\fi
\hypersetup{pdftitle={Dimension-Optimal Dynamic Regret in Exp-Concave Online Learning}}
\date{}

\begin{document}
\maketitle

\begin{abstract}
Tracking changing comparators is a central problem in online learning. For exp-concave losses, previous work established the $T^{1/3}P_T^{2/3}$ dependence on the horizon $T$ and Euclidean path length $P_T$ for fixed dimension, while the optimal cost of dimension remained unresolved. We resolve this question by establishing a sharp $d^{1/3}$ curved dimension factor, improving the $d^{2/3}$ factor of recent switching and interval reductions~\cite{WangEtAl2026Switching,WangZhaoZhou2026Interval}. To achieve this improvement, we develop a proper, full-information learner on compact convex subsets of $\mathbb R^d$ that adapts to the unknown path length and, under fixed dimension-independent diameter, gradient, and curvature parameters, guarantees
\[
\operatorname{Reg}_T(u_{1:T})
\le\widetilde O\!\left(d+\min\left\{
d^{1/3}T^{1/3}P_T^{2/3},\sqrt{TP_T},T\right\}\right).
\]
A matching lower bound on a diameter-one Euclidean ball establishes the minimax frontier up to logarithmic factors when $T\gtrsim d^2$, including the transition from the curved rate to larger path budgets. The lower bound already holds for losses that depend on only one coordinate per round, isolating an unavoidable cost of learning across directions. Our upper bound follows from square-root Gaussian KL control of comparator movement and a cancellation argument that pays for history restarts. The curved learner maintains $O(\log^2 T)$ Gaussian states, uses one gradient per round, and runs in time near-linear in $T$ under a Mahalanobis-projection oracle.
\end{abstract}

\begin{samepage}
\noindent\textbf{AI Usage.}
Nearly the entire research pipeline for this paper was carried out by \ResearchAgentSystem{}, powered by GPT-5.6 Sol. The system also conducted a Lean-backed article audit of the resulting manuscript. The authors subsequently reviewed and approved the mathematical claims, presentation, and formal artifacts, and take responsibility for the final manuscript. The complete Lean audit report and the system's technical report will be made public at a later date.
\par
\end{samepage}

\paragraph{Publication Intent and Recognition of Independent Work.} 
This preprint is released as a research output of our study of
automated scientific research, for public scrutiny and discussion.
As we do not intend to submit this manuscript to a journal or
conference, it has not undergone the full editorial refinement
and internal review we would normally require for submission.
The authors remain responsible for its contents and for correcting
any errors or inaccurate attribution.

We recognize that releasing these results may affect researchers
who have independently been pursuing the same questions, including
work that is not yet public.
If this applies to your work, we invite you to contact us at
\texttt{colombo.2026@outlook.com}.
A brief description of the overlap, together with any dated
materials you are comfortable sharing, would help us understand
its scope and chronology.
Such materials might include draft excerpts, presentation slides,
or revision records; a complete manuscript or unpublished proof
is not required to begin the conversation.

We will review these cases carefully and update citations,
acknowledgments, and accounts of independent or concurrent work
as appropriate, including revising priority claims where warranted.
Nonpublic materials will be treated confidentially and used only
to assess the overlap unless their authors authorize otherwise.
We will obtain explicit permission from the authors before
publicly discussing such materials or providing them to an
automated research system.


\section{Introduction}\label{sec:intro}
Online learning provides a principled framework for applications such as recommendation, portfolio selection, and tracking evolving networks, where decisions must be updated as data arrive~\cite{Hazan2022,HallWillett2013}. These sequential decision problems are commonly studied through online convex optimization (OCO): at round $t$, a learner chooses $w_t$ from a compact convex set $\cW\subset\mathbb R^d$ before observing a convex loss $f_t:\cW\to\mathbb R$. To evaluate its ability to track a changing benchmark, we measure the dynamic regret against a comparator sequence $u_{1:T}\in\cW^T$~\cite{Zinkevich2003}:
\begin{equation}\label{eq:intro-target}
\Reg_T(u_{1:T})=\sum_{t=1}^T\bigl[f_t(w_t)-f_t(u_t)\bigr].
\end{equation}
For general convex losses, adaptive methods achieve the optimal $O(\sqrt{T(1+P_T)})$ rate under fixed diameter and gradient bounds~\cite{ZhangEtAl2018,Cutkosky2020}, where the path length $P_T$ is defined as the total Euclidean distance traveled by the comparator sequence. Beyond this general convex setting, widely used prediction losses (e.g., squared and logistic losses on bounded prediction domains) possess additional curvature that permits faster learning~\cite{KorenLevy2015}. For such curved losses, Baby and Wang~\cite{BabyWang2021,BabyWang2022} established the sharper $T^{1/3}P_T^{2/3}$ temporal dependence, up to logarithmic factors for fixed dimension.

However, the model dimension also affects the cost of tracking, and its role depends on the geometry of the losses. Strong convexity supplies curvature in every direction and admits this fast rate without an explicit dimension factor under fixed Euclidean regularity parameters~\cite{WangEtAl2026Switching}. By contrast, exp-concavity permits curvature to concentrate in a subspace, suggesting a possible dimension cost for tracking across directions.

For exp-concave losses, Baby and Wang obtained curved-term factors of $d^{7/2}$~\cite{BabyWang2021} and $d^3$~\cite{BabyWang2022} in their respective settings, using $\ell_1$ path length.\footnote{Converting these published bounds to Euclidean path length via $\sum_{t=2}^T\|u_t-u_{t-1}\|_1\le\sqrt d\,P_T$ gives curved-term factors $d^{23/6}$ and $d^{10/3}$, respectively, under the original assumptions.} With Euclidean path length, Zhang et al.~\cite{ZhangZhaoSugiyama2025} obtain a factor $d$ through mixability. Recent switching-to-dynamic~\cite{WangEtAl2026Switching} and interval-to-dynamic~\cite{WangZhaoZhou2026Interval} reductions attain $d^{2/3}$. For proper learners, whose predictions remain in the decision set, these advances motivate the following question in full-information OCO:
\begin{researchquestion}
\raggedright
\textbf{Question.} \emph{What is the unavoidable dimension dependence of dynamic regret for general exp-concave losses?}
\end{researchquestion}

\paragraph{Dimension-optimal answer.}
We show that the curved dimension factor is $d^{1/3}$, attained by an implementable proper learner and required by a matching lower bound in the curved regime. We also characterize the transition to larger path budgets. With fixed dimension-independent regularity parameters and without knowing $P_T$, our learner guarantees
\begin{equation}\label{eq:intro-frontier}
\Reg_T(u_{1:T})\le\widetilde O\!\left[
d+\min\left\{d^{1/3}T^{1/3}P_T^{2/3},\sqrt{TP_T},T\right\}\right].
\end{equation}
The upper bound holds on every admissible compact convex domain. For $T\gtrsim d^2$, a matching lower bound establishes the minimax frontier over the normalized class of domains and losses (\Cref{cor:minimax}). \Cref{tab:comparison} summarizes the improvement over previous bounds in the curved regime.

\begin{table}[t]
\centering
\caption{Dynamic regret for proper exp-concave OCO on arbitrary compact convex domains, with Euclidean path length $P_T$. Diameter, gradient bounds, and curvature are fixed independently of $d$; logarithmic factors are suppressed. The displayed curved regime is $T\gtrsim d^2$ and $0\le P_T\lesssim T/d^2$. The matching lower bound is attained on a diameter-one ball.}
\label{tab:comparison}
\small
\begin{tabularx}{\textwidth}{@{}>{\raggedright\arraybackslash}p{0.25\textwidth}>{\centering\arraybackslash}p{0.24\textwidth}>{\centering\arraybackslash}p{0.18\textwidth}>{\raggedright\arraybackslash}X@{}}
\toprule
Work & Upper bound & Lower bound & Mechanism \\
\midrule
Zhang et al.~\cite{ZhangZhaoSugiyama2025} & $\widetilde O(d+dT^{1/3}P_T^{2/3})$ & --- & Gaussian fixed share; KL~projection \\
Wang et al.~\cite{WangEtAl2026Switching} & $\widetilde O(d+d^{2/3}T^{1/3}P_T^{2/3})$ & --- & Switching reduction with IRESET/ONS \\
Wang et al.~\cite{WangZhaoZhou2026Interval} & $\widetilde O(d+d^{2/3}T^{1/3}P_T^{2/3})$ & --- & Interval-to-dynamic reduction \\
\midrule
\multicolumn{1}{@{}>{\raggedright\arraybackslash}m{0.25\textwidth}}{\textbf{This Paper}\newline (\Cref{cor:minimax})} & \multicolumn{2}{>{\centering\arraybackslash}m{\dimexpr0.42\textwidth+2\tabcolsep\relax}}{$\boldsymbol{\widetilde\Theta(d+d^{1/3}T^{1/3}P_T^{2/3})}$} & \multicolumn{1}{>{\raggedright\arraybackslash}m{\dimexpr0.33\textwidth-6\tabcolsep\relax}@{}}{Sparse Gaussian histories; drawdown cancellation} \\
\bottomrule
\end{tabularx}
\end{table}

\begin{samepage}
\paragraph{An efficient upper bound.}
We develop \method{} from projected Gaussian histories~\cite{ZhangEtAl2026Portfolios}. This proper learner attains the $d^{1/3}$ curved factor without knowing the path length (\Cref{thm:curved}). The key is to charge comparator displacement through square-root Gaussian KL, reducing the movement coefficient from $d$ to $\sqrt d$; combining this control with rank-one smoothing yields the improved rate. Negative excess-regret prefixes pay for history restarts, so the learner needs only $O(\log^2 T)$ Gaussian states. Under a Mahalanobis-projection oracle, the curved algorithm uses one gradient per round, $\widetilde O(d^2)$ memory, and time near-linear in $T$. Combining it with adaptive projected-gradient learners yields the full upper bound in \Cref{cor:frontier}.
\par
\end{samepage}

\paragraph{A matching lower bound.}
We show that the $d^{1/3}$ factor is unavoidable, even for losses that depend on only one coordinate per round (\Cref{thm:lower}). Our construction places independent hidden signs across coordinates and time blocks, forcing the learner to repeatedly estimate changing coordinate means. It keeps the domain diameter, gradient bound, and exp-concavity constant independent of $d$, so the lower bound isolates an intrinsic cost of dimension. Varying the number of active coordinates also recovers the square-root and linear regimes, matching \eqref{eq:intro-frontier} up to logarithmic factors when $T\gtrsim d^2$.

\paragraph{Organization.}
\Cref{sec:related} reviews prior work, and \Cref{sec:results} defines the model and oracle assumptions. \Cref{sec:algorithm-overview} presents the algorithm, upper bounds, and their proof ideas. \Cref{sec:lower-main} establishes the lower bound and minimax characterization, followed by discussion in \Cref{sec:discussion}. \Cref{app:preliminaries} collects the preliminaries; \Cref{app:upper-proofs,sec:lower} provide the proofs for \Cref{sec:algorithm-overview,sec:lower-main}, including the finite-precision analysis.

\section{Related work}\label{sec:related}

\paragraph{Path-length regret and other measures of change.}
Zinkevich's projected-gradient analysis~\cite{Zinkevich2003} established path length as a measure of comparator movement. Ader~\cite{ZhangEtAl2018} aggregates learning rates to adapt to an unknown path budget, while Cutkosky~\cite{Cutkosky2020} connects dynamic regret with parameter-free and strongly adaptive learning. These methods supply the general convex benchmark and motivate the projected-gradient branch of our full frontier. Hall and Willett~\cite{HallWillett2013} instead measure deviation from a family of dynamical models, allowing a predictable moving path to have small tracking complexity. Besbes, Gur, and Zeevi~\cite{BesbesEtAl2015} study nonstationary stochastic optimization under variation of the loss sequence. Their benchmark and feedback model distinguish that problem from simultaneous regret guarantees for every comparator path. More recently, Jacobsen et al.~\cite{JacobsenEtAl2025} reduce dynamic regret to static regret in a suitable reproducing kernel Hilbert space, making the choice of comparator geometry part of the reduction.

\paragraph{Curvature, mixability, and variation-bounded prediction.}
Exp-concavity supports the $O(d\log T)$ static guarantee of ONS~\cite{HazanAgarwalKale2007}; mixability more broadly permits aggregation through a suitable prediction rule~\cite{Vovk2001}. Baby and Wang~\cite{BabyWang2021,BabyWang2022} derive curved dynamic rates by analyzing the offline comparator's KKT conditions. Their proper strongly convex bound has factor $d^{1/3}$ in terms of $P_T^{(1)}=\sum_t\|u_t-u_{t-1}\|_1$, and their proper exp-concave result treats a box domain. Since $P_T\le P_T^{(1)}\le\sqrt d\,P_T$, dimension comparisons must keep the path norm and regularity normalization explicit. Zhang et al.~\cite{ZhangZhaoSugiyama2025} replace KKT analysis with mixability and extend Gaussian aggregation to general exp-concave OCO. Zhang et al.~\cite{ZhangEtAl2026Portfolios} subsequently give a tractable implementation with projected Gaussian histories. Their Theorem~10 has a curved factor $d$ under bounded $\ell_1$ diameter, bounded $\ell_\infty$ gradients, and $\ell_1$ path length. Our advance combines sparse histories, square-root KL movement control, and drawdown cancellation to obtain the sharp Euclidean dimension dependence. The scalar variation problem also connects to online nonparametric regression through sequential entropy~\cite{RakhlinSridharan2014} and adaptive forecasting of total-variation-bounded signals~\cite{BabyWang2019}.

\paragraph{Adaptive histories and switching reductions.}
Fixed share tracks a changing best expert by continually redistributing weight~\cite{HerbsterWarmuth1998}. Follow-the-Leading-History~\cite{HazanSeshadhri2007} and strongly adaptive reductions~\cite{DanielyEtAl2015} organize learners born at different times; coin-betting methods improve interval guarantees with comparable time overhead~\cite{JunEtAl2017}. Gokcesu and Gokcesu~\cite{GokcesuGokcesu2021} develop near-linear tracking for mixable or exp-concave losses. RESET~\cite{PasterisEtAl2024} obtains optimal switching regret for all segmentations simultaneously, and IRESET~\cite{YangEtAl2026} extends logarithmic switching guarantees to curved OCO. Building on this interface, Wang et al.~\cite{WangEtAl2026Switching} approximate a moving comparator by an unbiased random sequence with controlled variance and expected switches. Their exp-concave instantiation gives a $d^{2/3}$ factor. Wang, Zhao, and Zhou~\cite{WangZhaoZhou2026Interval} obtain the same dimension factor through an interval-to-dynamic reduction on arbitrary compact convex domains; their Theorem~3 gives the bound compared in \Cref{tab:comparison}. Our improvement charges comparator movement directly through Gaussian potential.

\paragraph{Efficient exp-concave optimization and alternative reductions.}
Mahalanobis projection is a principal cost of second-order online algorithms. Mhammedi and Gatmiry~\cite{MhammediGatmiry2023} reduce this cost through quasi-Newton steps, and LightONS~\cite{WangZhaoZhou2026} preserves optimal static regret with total time $O(d^2T+d^\omega\sqrt{T\log T})$, where $\omega$ is the matrix-multiplication exponent. Transferring such projection savings to the Gaussian states in our algorithm would strengthen the computational result. Xie et al.~\cite{XieEtAl2026} use a discounted-to-dynamic reduction for FTRL-based linear and logistic regression, illustrating another route that exploits the structure of particular curved losses. Our arithmetic bound counts the projection cost explicitly and applies to any domain admitting the stated oracle. Classical rank-one eigensystem modification~\cite{BunchNielsenSorensen1978,GuEisenstat1994} motivates a possible acceleration; our conditional computational statement specifies the full spectral-update interface it requires.

\section{Problem setup}\label{sec:results}

At round $t\in[T]$, the learner chooses $w_t\in\cW$, then observes a differentiable convex loss $f_t:\cW\to\R$ and incurs $f_t(w_t)$.  For $u_{1:T}\in\cW^T$, define
\[
 \Reg_T(u_{1:T})=\sum_{t=1}^T[f_t(w_t)-f_t(u_t)],
 \qquad
 P_T(u)=\sum_{t=2}^T\norm{u_t-u_{t-1}}_2.
\]

The learner may depend on $d,T,D,G,\eta$ and the domain, but not on the comparator or its path length. Full information permits evaluation of the revealed loss and its gradient at additional points. We count these evaluations separately from projection and arithmetic costs. The curved learner queries one gradient at its own prediction; the outer ensemble used for the full frontier evaluates its base learners after observing $f_t$. All upper bounds hold for every realized admissible loss sequence. The lower bound uses an oblivious distribution of losses and applies to randomized learners. We write $\widetilde O$ for bounds hiding only logarithmic factors in $T$ when regularity parameters are fixed, and use natural logarithms unless a base is indicated.

\begin{assumption}[Geometry and curvature]\label[assumption]{ass:model}
The set $\cW\subset\R^d$ is nonempty, compact, and convex, with Euclidean diameter at most $D>0$.  Every $f_t$ is $\eta$-exp-concave on $\cW$ and satisfies
$\norm{\nabla f_t(w)}_2\le G$ for all $w\in\cW$.
\end{assumption}

\Cref{app:preliminaries} recalls the definition of exp-concavity, its curvature criterion, and representative losses satisfying this assumption. The analysis uses the following curvature and logarithmic parameters:
\begin{equation}\label{eq:gamma}
 \gamma=\min\left\{\frac\eta2,\frac1{4GD}\right\},
 \qquad
 \lambda=\frac\gamma2,
 \qquad
 \Lambda_T=1+\log(2T).
\end{equation}
The degenerate case $GD=0$ has zero regret and will be omitted.

The statistical assumptions above specify the loss class. To state the algorithm's computational cost on a general domain, we also make explicit the projection operation used to keep its predictions~feasible.

\begin{assumption}[Exact-arithmetic projection oracle]\label[assumption]{ass:exact-oracle}
The oracle returns an exact solution to
\[
 \min_{w\in\cW}(w-a)^\top A(w-a)
\]
for every $a\in\R^d$ and $A\succ0$, at arithmetic cost $\mathsf C_{\rm proj}(d)$.
\end{assumption}

This oracle measures the cost of enforcing the domain constraint in the geometry of each Gaussian state. It enters the upper-bound implementation; the lower bound applies to every learner, regardless of its computational resources.

The arithmetic bounds use the dense spectral model: a symmetric eigendecomposition and reconstruction are charged $O(d^3)$ operations. This is an idealized real-arithmetic convention. \Cref{sec:finite-precision} gives a separate rational implementation and its bit cost under an explicit approximate-projection oracle.

\section{Algorithm and upper bounds}\label{sec:algorithm-overview}
We first construct the curved learner, \method{}, and state its regret and computational guarantees. We then explain why its movement analysis gives the $d^{1/3}$ factor and how to extend the guarantee to all path budgets.

\subsection{A common rank-one surrogate}
After the master predicts $w_t\in\cW$ and observes $g_t=\nabla f_t(w_t)$, all histories use
\begin{equation}\label{eq:overview-surrogate}
\widetilde f_t(v)=\ip{g_t}{v-w_t}+\frac\gamma2\ip{g_t}{v-w_t}^2,
\qquad
\gamma=\min\left\{\frac\eta2,\frac1{4GD}\right\}.
\end{equation}
This surrogate dominates the original one-step regret and is itself $\gamma/2$-exp-concave.  Its Hessian is rank one, which is essential: smoothing a comparator by $\N(u_t,\tau^2I)$ costs exactly
\[
\frac{\gamma\tau^2}{2}\|g_t\|_2^2,
\]
with no ambient trace factor $d$. Thus the same surrogate both transfers regret back to the original losses and keeps the smoothing error independent of dimension; \Cref{lem:surrogate} proves the required domination and exp-concavity properties.

\subsection{Sparse Gaussian histories}
For dyadic scales $\tau_k=D2^{-k}$, $k=0,\ldots,\lfloor\log_2 T\rfloor$, we run a deterministic logarithmic active set of histories.  A history born at time $i$ and scale $k$ maintains a single Gaussian $\N(m_t,\Sigma_t)$ satisfying
\[
m_t\in\cW,
\qquad
\tau_k^2I\preceq\Sigma_t\preceq D^2I.
\]
The exp-concave master averages the means, so the actual prediction is proper.  Exponential tilting by \eqref{eq:overview-surrogate} preserves Gaussianity.  We then floor the covariance at $\tau_k^2I$ and project the mean onto $\cW$ in the resulting Mahalanobis metric.  Both operations decrease KL from every comparator Gaussian $\N(u,\tau_k^2I)$.

Birth times let the learner adapt to changes in the comparator, while the scale grid adapts to its unknown path length. \Cref{alg:sparsegauss} summarizes one round of the resulting learner. The active-start schedule and master weights are specified in \Cref{app:histories}, and the Gaussian updates in \Cref{app:gaussian}.

\begin{algorithm}[htbp]
\caption{\method: sparse Gaussian histories for the curved branch}
\label{alg:sparsegauss}
\begin{algorithmic}[1]
\Require Horizon $T$, scales $\tau_k=D2^{-k}$, projection oracle, prior mean $w_0\in\cW$
\State Initialize the streaming active-start schedule and one Gaussian $\N(w_0,D^2I)$ for each newborn scale
\For{$t=1,\ldots,T$}
  \State Collect the means $x_t^{i,k}$ of all active histories and predict $w_t=\sum_{i,k}p_t^{i,k}x_t^{i,k}$
  \State Observe $f_t$, evaluate $g_t=\nabla f_t(w_t)$, and form the common surrogate \eqref{eq:overview-surrogate}
  \State Update master weights by exponential weighting on $\widetilde f_t(x_t^{i,k})$
  \For{each active history $(i,k)$}
    \State Exponentially tilt its Gaussian by $\exp[-(\gamma/2)\widetilde f_t]$
    \State Floor the covariance at $\tau_k^2I$ and Mahalanobis-project the mean onto $\cW$
  \EndFor
  \State Prune expired starts; insert the new start $t+1$ at every scale and renormalize
\EndFor
\end{algorithmic}
\end{algorithm}

\subsection{Regret and computational guarantees}\label{sec:upper-guarantees}
The learner above runs all candidate scales simultaneously, so it adapts to the unknown path length. We state the resulting rate with fixed regularity parameters to make its dependence on $d$, $T$, and $P_T$ explicit. Recall that $\Lambda_T=1+\log(2T)$; the $O(\cdot)$ notation below hides constants, while retaining all logarithmic factors.

\begin{theorem}[Curved dynamic regret]\label[theorem]{thm:curved}
Under \Cref{ass:model,ass:exact-oracle}, fix $D,G,\eta>0$ independently of $d$ and $T$. The proper full-information learner \method{} does not require $P_T(u)$ and, for every comparator path with $P_T(u)\le cT/(\sqrt d\,\Lambda_T^2)$, guarantees
\begin{equation}\label{eq:curved-summary}
 \Reg_T(u_{1:T})
 \le O\!\left(d\Lambda_T^2
 +d^{1/3}\Lambda_T^{4/3}T^{1/3}P_T(u)^{2/3}\right).
\end{equation}
Here $c>0$ and the constant hidden by $O(\cdot)$ depend only on $D,G,\eta$.
The learner uses one gradient query per round, requires $O((d^3+\mathsf C_{\rm proj}(d))\log^2(2T))$ operations per round in the dense spectral model, and stores $O(d^2\log^2(2T))$ real numbers. If an additional primitive performs the complete rank-one update and covariance flooring in $O(d^2)$ operations, including reconstruction in the original coordinates, the $d^3$ term becomes $d^2$.
\end{theorem}

The $d^{1/3}$ factor follows by optimizing the movement--smoothing tradeoff. The path condition ensures that the balancing scale does not exceed the domain diameter. \Cref{thm:curved-detailed} gives the full dependence on $D,G,\eta$ and the clipped-scale bound valid for every path. The next corollary removes the path restriction by adding projected-gradient learners.

At any time there are $O(\log T)$ active starts and $O(\log T)$ scales, hence $O(\log^2T)$ Gaussian states. Each state requires one rank-one precision update, covariance flooring, and one Mahalanobis projection. Dense spectral computation gives the stated $O(d^3)$ cost. Classical rank-one eigensystem methods~\cite{BunchNielsenSorensen1978,GuEisenstat1994} motivate the additional update primitive; the conditional $O(d^2)$ cost includes the full change of basis and covariance reconstruction. Storing one Gaussian per history keeps the total state count logarithmic in the number of starts and scales.

These state counts give the per-round cost in \Cref{thm:curved}, hence $O(T(d^3+\mathsf C_{\rm proj}(d))\log^2(2T))$ arithmetic operations over the horizon. A polynomial-bit implementation is given in \Cref{sec:finite-precision}.

\subsection{Why the dimension factor improves}\label{sec:upper-analysis}
The proof of \Cref{thm:curved} has three steps: control displacement within one history, pay for history restarts, and optimize the smoothing scale. The central obstacle is that a moving comparator prevents the usual KL potential from telescoping directly.
For a fixed history and scale, let
\[
F_t(u)=\KL\!\left(\N(u,\tau^2I)\,\|\,\N(m_t,\Sigma_t)\right).
\]
For a fixed comparator, the one-step inequality in \Cref{lem:one-step} telescopes. Comparator movement introduces an extra term, which \Cref{lem:root-movement} controls through the square root of the potential:
\[
\sqrt{F_{t+1}(v)}\le \sqrt{F_{t+1}(u)}+\frac{\|v-u\|_2}{\sqrt2\,\tau}.
\]
A single history has no time-uniform fixed-share envelope.  Let $r_t$ be its surrogate excess against the Gaussian comparator and $S_q=\sum_{t=i}^q r_t$.  If $S_q\ge-H$ throughout an interval, the shifted energy
\[
E_t=F_t(u_t)+\frac\gamma2(S_{t-1}+H)
\]
is nonnegative and controls movement throughout the interval (\Cref{lem:drawdown}). If $S_q<-H$ for the first time, that negative prefix pays the master cost and the recursion restarts on the remaining suffix. The streaming-history property ensures that every interval exposes a fresh history in its first half. Deleting negative nodes leaves $O(\log T)$ positive nodes per component. This cancellation gives the low-variation interval guarantee in \Cref{lem:low-variation}; it is what allows sparse histories to retain the improved movement cost.

Finally, partition the comparator path into low-variation intervals and sum their costs. Adding the rank-one smoothing error yields the fixed-scale bound of \Cref{thm:fixed-scale}:
\[
\Reg_T(u_{1:T})
\lesssim
\frac{d\,\polylog T}{\gamma}
+\frac{\sqrt d\,\polylog T}{\gamma\tau}P_T
+\gamma G^2T\tau^2.
\]
The second term is the cost of movement and the third is the smoothing error. Balancing them produces $d^{1/3}T^{1/3}P_T^{2/3}$; \Cref{lem:endpoints} handles scales whose unconstrained optimizer falls outside $[D/T,D]$. This completes the route from the local potential inequalities to \Cref{thm:curved}.

\subsection{Extension to all path budgets}\label{sec:full-upper}
For larger path budgets, projected-gradient learners can give a smaller regret than the curved branch. We combine \method{} with a logarithmic grid of their learning rates and use exp-concavity to compete with the best branch. This yields the next corollary, with additional gradient and projection calls because each base learner is evaluated at its own prediction.

\begin{corollary}[Full upper frontier]\label[corollary]{cor:frontier}
Under \Cref{ass:model,ass:exact-oracle}, fix $D,G,\eta>0$ independently of $d$ and $T$. A proper full-information learner that does not require $P_T(u)$ guarantees
\begin{equation}\label{eq:full-frontier}
 \Reg_T(u_{1:T})
 \le O\!\left(d\Lambda_T^2+\min\left\{
 d^{1/3}\Lambda_T^{4/3}T^{1/3}P_T(u)^{2/3},\,
 \sqrt{T(1+P_T(u))},\,T\right\}\right).
\end{equation}
for every comparator path. The constant hidden by $O(\cdot)$ depends only on $D,G,\eta$. The learner uses $O(\log T)$ loss evaluations, gradient evaluations, and Euclidean projection calls per round in addition to the curved branch's internal work.
\end{corollary}

The first branch in \eqref{eq:full-frontier} agrees with \eqref{eq:curved-summary}; the other branches control larger path budgets and unrestricted movement. The logarithmic aggregation cost is absorbed by the additive term. \Cref{cor:frontier-detailed} records the full parameter dependence, and \Cref{app:frontier} proves it from \Cref{thm:curved-detailed} and the projected-gradient bound in \Cref{lem:dogd}.

\section{Lower bounds and minimax optimality}\label{sec:lower-main}
We now ask whether any learner can improve the dimension dependence of the upper bound. The construction below uses the same Euclidean path length and fixes the diameter, gradient bound, and exp-concavity constant independently of $d$. Thus any remaining dimension factor reflects the difficulty of tracking the comparator.

\subsection{A matching lower bound}
Each loss reveals information about only one coordinate. Although the learner observes the entire loss function, the other coordinates remain unobserved on that round. The following theorem turns this restriction into a lower bound that applies even to randomized learners.

\begin{theorem}[Matching normalized lower frontier]\label[theorem]{thm:lower}
There are universal constants $c,c_0>0$ such that the following holds.  Let
\[
 \cW_d=\{w\in\R^d:\norm w_2\le1/2\}
\]
and let the adversary use only the losses
\begin{equation}\label{eq:hard-loss}
 f_{j,y}(w)=\frac12(w_j-y)^2,
 \qquad j\in[d],\quad y\in\{-1,+1\}.
\end{equation}
This class has diameter $D=1$, gradient bound $G\le3/2$, and exp-concavity coefficient at least $4/9$, uniformly in $d$.  If $T\ge c_0d^2$, then for every $P\ge0$,
\begin{equation}\label{eq:lower-frontier}
 \inf_{\mathsf A}\sup_{\substack{f_{1:T}\text{ of the form }\eqref{eq:hard-loss}\\
 u_{1:T}\in\cW_d^T,\ P_T(u)\le P}}
 \E_{\mathsf A}\Reg_T(u_{1:T})
 \ge
 c\left[
 d+\min\left\{
 d^{1/3}T^{1/3}P^{2/3},\sqrt{TP},T
 \right\}
 \right].
\end{equation}
\end{theorem}

\paragraph{Why the construction is hard.}
In each time block, the comparator encodes independent hidden signs across a chosen number $m$ of coordinates. Noisy labels reveal too little information to estimate those signs accurately before the next block. \Cref{lem:two-point} quantifies the one-coordinate prediction error, and \Cref{lem:effective-lower} combines these errors while respecting the Euclidean norm and path constraints.

Using all $d$ coordinates gives the curved factor $d^{1/3}$. As the path budget grows, choosing fewer active coordinates yields the square-root branch, followed by linear saturation. The stationary construction in \Cref{lem:static-lower} supplies the additive $d$ term even when the comparator never moves. The complete construction and the choices of $m$ are given in \Cref{sec:lower}.

\subsection{The minimax frontier}
The upper bound applies to every admissible compact convex domain, whereas the lower bound already holds on a Euclidean ball. Together they determine the worst-case regret over the normalized class. The next corollary combines the two results and records the condition on the horizon under which they match.

\begin{corollary}[Computational minimax closure, up to logarithms]\label[corollary]{cor:minimax}
For $D=1$, $G=3/2$, and $\eta=4/9$, take the worst case over compact convex domains of diameter at most one and losses satisfying \Cref{ass:model}. For horizons $T\gtrsim d^2$, the proper full-information minimax dynamic regret with path budget $P_T$ has, up to polylogarithmic factors, the frontier
\[
 \boxed{
 d+\min\{d^{1/3}T^{1/3}P_T^{2/3},\sqrt{TP_T},T\}.
 }
\]
With a Mahalanobis-projection oracle, the upper frontier is attained in near-linear time in $T$.
\end{corollary}

\begin{proof}
The lower bound follows from \Cref{thm:lower}, since its ball and coordinate-square losses belong to the stated class. For the upper bound, \Cref{cor:frontier} gives the curved branch, $\sqrt{T(1+P_T)}$, and the linear cap, together with an additive $d$, up to logarithms. For $P_T\ge1$, the square-root branch is within a constant factor of $\sqrt{TP_T}$. For $0<P_T<1$ and $T\ge d^2$, the ratio of the curved term to $\sqrt{TP_T}$ is $(d^2P_T/T)^{1/6}\le1$, so the curved term is already the smaller branch. The case $P_T=0$ follows directly from the static term. This proves the displayed normalized frontier.
\end{proof}

\section{Discussion and conclusion}\label{sec:discussion}
Sparse Gaussian histories attain the sharp $d^{1/3}$ curved dimension factor, and coordinate-square losses show why this cost is unavoidable. The upper analysis couples rank-one smoothing with square-root KL movement control; negative excess-regret prefixes pay for the history restarts needed to maintain only logarithmically many active states. Combining the curved learner with projected-gradient learners gives the full minimax frontier when $T\gtrsim d^2$.

The arithmetic result depends on the cost of Mahalanobis projection onto the domain. Projection deferral, as in LightONS~\cite{WangZhaoZhou2026}, is a natural direction for reducing that cost while preserving the comparator-KL inequality. A second question is whether one gradient at the final prediction can support the entire frontier: the current curved learner has this property, whereas the full-frontier ensemble evaluates its bases separately. The role of more restrictive domain or loss geometry in reducing the effective dimension also remains open.

\clearpage
\appendix
\section{Preliminaries}\label{app:preliminaries}
We collect the curvature facts used in the upper and lower bounds, together with examples that clarify the loss class in \Cref{ass:model}.

\subsection{Exp-concavity and directional curvature}\label{app:exp-concavity}
A function $f:\cW\to\R$ is $\eta$-exp-concave, for $\eta>0$, if $w\mapsto e^{-\eta f(w)}$ is concave on the convex set $\cW$~\cite{Hazan2022,HazanAgarwalKale2007}. For differentiable $f$, the supporting-hyperplane inequality for this concave function gives
\begin{equation}\label{eq:prelim-tangent}
 e^{-\eta f(v)}\le e^{-\eta f(w)}
 \bigl(1-\eta\ip{\nabla f(w)}{v-w}\bigr),
 \qquad v,w\in\cW.
\end{equation}
This is the first-order inequality used to obtain surrogate domination in \Cref{lem:surrogate}. Exp-concavity with parameter $\eta$ also implies exp-concavity with every parameter $\eta'\in(0,\eta]$, since $x\mapsto x^{\eta'/\eta}$ is increasing and concave on $(0,\infty)$.

For a $C^2$ loss on an open convex domain, the equivalent curvature criterion is
\begin{equation}\label{eq:prelim-hessian}
 \nabla^2 f(w)\succeq\eta\nabla f(w)\nabla f(w)^\top.
\end{equation}
Differentiating the exponential gives
\[
 \nabla^2 e^{-\eta f(w)}
 =\eta e^{-\eta f(w)}
 \bigl(\eta\nabla f(w)\nabla f(w)^\top-\nabla^2 f(w)\bigr).
\]
The same criterion suffices on a compact convex subset when the loss has a smooth extension to a neighborhood; on a lower-dimensional domain it is enough to check directions in its affine hull. We use this criterion to verify the smooth examples below; the main results assume differentiability and exp-concavity as stated in \Cref{ass:model}.

Strong convexity requires curvature in every direction: in the smooth case, $\nabla^2 f(w)\succeq\sigma I$ for some $\sigma>0$. If also $\|\nabla f(w)\|_2\le G$ with $G>0$, then \eqref{eq:prelim-hessian} holds for $\eta=\sigma/G^2$. Exp-concavity allows the Hessian to be rank one, as the examples below illustrate, so it need not imply strong convexity in the ambient space.

\subsection{Representative exp-concave losses}\label{app:examples}
For a linear prediction $z=x^\top w$ and a scalar loss $\ell(z)$, the gradient and Hessian of $f(w)=\ell(x^\top w)$ are $\ell'(z)x$ and $\ell''(z)xx^\top$. Thus the scalar condition $\ell''(z)\ge\eta(\ell'(z))^2$ verifies \eqref{eq:prelim-hessian} without a dimension~factor.

\paragraph{Squared prediction loss.}
For $f(w)=\tfrac12(x^\top w-y)^2$ on a domain where $|x^\top w-y|\le B$ with $B>0$, the scalar derivatives are $\ell'(z)=z-y$ and $\ell''(z)=1$. The loss is therefore $1/B^2$-exp-concave. Its gradient norm is at most $B\|x\|_2$. The coordinate-square losses in \Cref{thm:lower} take $x=e_j$ and $B=3/2$, giving the coefficient $4/9$ used there.

\paragraph{Logistic prediction loss.}
For $y\in\{-1,1\}$ and $f(w)=\log(1+e^{-yx^\top w})$, the derivatives satisfy
\[
 \ell'(z)=\frac{-y}{1+e^{yz}},\qquad
 \ell''(z)=\frac{e^{yz}}{(1+e^{yz})^2},\qquad
 \frac{\ell''(z)}{(\ell'(z))^2}=e^{yz}.
\]
On a domain with $|x^\top w|\le B$, this loss is $e^{-B}$-exp-concave and has gradient norm at most $\|x\|_2$. The bounded prediction range supplies a uniform positive exp-concavity parameter.

\paragraph{Negative logarithmic return.}
For a return vector $r$ and a domain on which $r^\top w>0$, the portfolio loss $f(w)=-\log(r^\top w)$ is $1$-exp-concave because $e^{-f(w)}=r^\top w$ is affine~\cite{Hazan2022}. If $r^\top w\ge b>0$, its gradient norm is at most $\|r\|_2/b$. The lower bound on returns is needed for the uniform gradient bound in \Cref{ass:model}.

\subsection{Aggregation under exp-concavity}\label{app:aggregation}
For predictions $w_i\in\cW$ and weights $p_i\ge0$ summing to one, convexity of the domain makes $\bar w=\sum_i p_iw_i$ feasible. Concavity of $e^{-\eta f}$ then yields
\begin{equation}\label{eq:prelim-aggregation}
 f(\bar w)\le-\frac1\eta\log\left(\sum_i p_i e^{-\eta f(w_i)}\right).
\end{equation}
This inequality is the aggregation step for both the surrogate master in \Cref{app:histories} and the outer ensemble in \Cref{app:frontier}. The former uses the surrogate's exp-concavity parameter $\lambda$, whereas the latter uses $\eta$ for the original losses.

\section{Proofs for Section~\ref{sec:algorithm-overview}}\label{app:upper-proofs}
We collect the supporting lemmas with their proofs, then prove the parameter-explicit versions of \Cref{thm:curved,cor:frontier}. The final subsection establishes the finite-precision implementation and its anytime extension.

\subsection{Useful lemmas}\label{app:upper-lemmas}
The argument proceeds from surrogate domination and history aggregation to Gaussian movement control and restart cancellation. The path-partition and projected-gradient lemmas then supply the two branches of the full frontier. We also record the numerical robustness statement used in the implementation theorem.

\subsubsection{Rank-one surrogate: domination and curvature}\label{app:surrogate}

After the algorithm forms its proper prediction $w_t$, let
$g_t=\nabla f_t(w_t)$ and define
\begin{equation}\label{eq:surrogate}
 \widetilde f_t(v)
 =\ip{g_t}{v-w_t}
 +\frac\gamma2\ip{g_t}{v-w_t}^2.
\end{equation}

The first step is to justify analyzing the quadratic surrogate instead of the original losses. The next lemma gives both ingredients needed later: regret domination for the final guarantee and exp-concavity for the master aggregation.

\begin{lemma}[Domination and surrogate exp-concavity]\label[lemma]{lem:surrogate}
For every $v\in\cW$,
\begin{equation}\label{eq:domination}
 f_t(w_t)-f_t(v)
 \le \widetilde f_t(w_t)-\widetilde f_t(v)
 =-\widetilde f_t(v).
\end{equation}
Moreover, $\widetilde f_t$ is $\lambda$-exp-concave on $\cW$.
\end{lemma}

\begin{proof}
Because $2\gamma\le\eta$, the loss is also $2\gamma$-exp-concave.  Put
$a=\ip{g_t}{v-w_t}$.  Concavity of $e^{-2\gamma f_t}$~gives
\[
 e^{-2\gamma f_t(v)}
 \le e^{-2\gamma f_t(w_t)}(1-2\gamma a).
\]
The choice of $\gamma$ implies $|2\gamma a|\le2\gamma GD\le1/2$.
For $|x|\le1/2$, $-\log(1-x)\ge x+x^2/4$.  Taking logarithms yields
\[
 f_t(v)-f_t(w_t)
 \ge a+\frac\gamma2a^2=\widetilde f_t(v),
\]
which is \eqref{eq:domination}.

For exp-concavity, write $s=\ip{g_t}{v-w_t}$.  Then
\[
 \nabla\widetilde f_t(v)=(1+\gamma s)g_t,
 \qquad
 \nabla^2\widetilde f_t(v)=\gamma g_tg_t^\top.
\]
Since $|\gamma s|\le1/4$,
\[
 \nabla^2\widetilde f_t(v)
 \succeq\frac\gamma2(1+\gamma s)^2g_tg_t^\top
 =\lambda\nabla\widetilde f_t(v)\nabla\widetilde f_t(v)^\top.
\]
This is the Hessian criterion for $\lambda$-exp-concavity.
\end{proof}

\subsubsection{Sparse histories and the exp-concave master}\label{app:histories}

The deterministic streaming schedule follows Hazan and Seshadhri~\cite{HazanSeshadhri2007}. Its role is to expose a fresh history early in any interval while keeping only logarithmically many states.
For an integer $i\ge1$, let $\nu_2(i)$ be the largest $k\ge0$ such that $2^k$ divides $i$, and define its lifetime
\begin{equation}\label{eq:lifetime}
 \ell(i)=2^{\nu_2(i)+2}.
\end{equation}
At time $t$, keep
\begin{equation}\label{eq:active-set}
 \cS_t=\{i\in[t]:t-i\le\ell(i)\}.
\end{equation}

\begin{lemma}[Streaming-history properties]\label[lemma]{lem:schedule}
The sets \eqref{eq:active-set} satisfy:
\begin{enumerate}[label=(H\arabic*),leftmargin=*]
\item $t\in\cS_t$ and $|\cS_t|\le c_H\log(2t)$ for a universal $c_H$;
\item once $i$ leaves the active set, it never returns;
\item for every interval $[r,s]\subseteq[T]$, there is an
$i\in\cS_s\cap[r,\lfloor(r+s)/2\rfloor]$.
\end{enumerate}
Consequently, the history born at such an $i$ is active on every round from $i$ through $s$.
\end{lemma}

\begin{proof}
The first assertion in (H1) and (H2) are immediate from \eqref{eq:active-set}.  For the size bound, fix $k$.  Integers with $\nu_2(i)=k$ are separated by $2^{k+1}$, and an active such integer lies in an interval of length $2^{k+2}$.  Hence only a universal constant number occur for each $k$, and $k\le\lfloor\log_2t\rfloor$.

For (H3), put $\Delta=s-r$.  If $\Delta<2$, the claim is immediate.  Otherwise choose the largest $2^k\le\Delta/2$.  The interval $[r,(r+s)/2]$ has length at least $2^k$, so it contains an integer $i$ divisible by $2^k$.  Thus $\ell(i)\ge2^{k+2}>\Delta\ge s-i$, and $i\in\cS_s$.  Property (H2) ensures activity throughout $[i,s]$.
\end{proof}

Define the smoothing grid and its prior by
\[
 K=\lfloor\log_2T\rfloor,
 \qquad
 \tau_k=D2^{-k},\quad k=0,\ldots,K,
 \qquad
 q_k=\frac1{K+1}.
\]
For every active pair $(i,k)$, maintain one Gaussian history learner
$\mathcal E_{i,k}$.  At round $t$, it predicts a point $x_t^{i,k}\in\cW$.
Initialize $p_1^{1,k}=q_k$. The master has weights $p_t^{i,k}$ on active pairs and predicts
\begin{equation}\label{eq:master-prediction}
 w_t=\sum_{k=0}^K\sum_{i\in\cS_t}p_t^{i,k}x_t^{i,k}\in\cW.
\end{equation}
After observing \eqref{eq:surrogate}, perform the exponential update
\[
 \widehat p_{t+1}^{i,k}
 =\frac{p_t^{i,k}e^{-\lambda\widetilde f_t(x_t^{i,k})}}
 {\sum_{j\in\cS_t,\ell}p_t^{j,\ell}e^{-\lambda\widetilde f_t(x_t^{j,\ell})}}.
\]
Before round $t+1$, give total weight $1/(t+1)$ to the new start $t+1$, split according to $q_k$, multiply surviving old weights by $t/(t+1)$, delete starts outside $\cS_{t+1}$, and renormalize.

The schedule guarantees that a suitable history exists on any interval. We next bound the price of following that history. A guarantee for every prefix, rather than only its full lifetime, is needed because the cancellation argument may stop a history at its first large negative excess-regret prefix.

\begin{lemma}[Prefix overhead to every surviving history]\label[lemma]{lem:meta}
If $\mathcal E_{i,k}$ is active through round $b\ge i$, then every prefix $[i,q]\subseteq[i,b]$ satisfies
\begin{equation}\label{eq:meta-overhead}
 \sum_{t=i}^q
 [\widetilde f_t(w_t)-\widetilde f_t(x_t^{i,k})]
 \le a_k,
 \qquad
 a_k=\frac1\lambda\log\frac{T+1}{q_k}.
\end{equation}
For the uniform scale prior, write
\begin{equation}\label{eq:a}
 a=\lambda^{-1}\log((T+1)(K+1)).
\end{equation}
\end{lemma}

\begin{proof}
By exp-concavity of the surrogate and \eqref{eq:master-prediction},
\[
 e^{-\lambda\widetilde f_t(w_t)}
 \ge\sum_{j,\ell}p_t^{j,\ell}e^{-\lambda\widetilde f_t(x_t^{j,\ell})}.
\]
Rearranging this inequality gives the one-step comparison
\begin{equation}\label{eq:meta-one-step}
 \lambda[\widetilde f_t(w_t)-\widetilde f_t(x_t^{i,k})]
 \le\log\frac{\widehat p_{t+1}^{i,k}}{p_t^{i,k}}.
\end{equation}
At birth, pruning renormalization can only increase mass, so
$p_i^{i,k}\ge q_k/i$.  For $t>i$, survival and renormalization give
$p_t^{i,k}\ge((t-1)/t)\widehat p_t^{i,k}$.  Summing
\eqref{eq:meta-one-step},
\begin{align*}
 \lambda\sum_{t=i}^q
 [\widetilde f_t(w_t)-\widetilde f_t(x_t^{i,k})]
 &\le -\log p_i^{i,k}+\log\widehat p_{q+1}^{i,k}
 +\sum_{t=i+1}^q\log\frac{\widehat p_t^{i,k}}{p_t^{i,k}}\\
 &\le \log\frac{i}{q_k}+0+\sum_{t=i+1}^q\log\frac t{t-1}\\
 &=\log\frac q{q_k}
 \le\log\frac{T+1}{q_k}.
\end{align*}
The terminal term is nonpositive because
$\widehat p_{q+1}^{i,k}\le1$.
\end{proof}

\subsubsection{One Gaussian history: update, feasibility, and KL control}\label{app:gaussian}

Fix a scale $\tau\in[D/T,D]$.  A history maintains
$P_t=\N(m_t,\Sigma_t)$ with
\begin{equation}\label{eq:cov-range}
 m_t\in\cW,
 \qquad
 \tau^2I_d\preceq\Sigma_t\preceq D^2I_d.
\end{equation}
At birth, $P_i=\N(w_0,D^2I_d)$ for a fixed $w_0\in\cW$.

\paragraph{Closed-form rank-one Gaussian update.}

Let $A_t=\Sigma_t^{-1}$.  Multiplying the Gaussian density by
$e^{-\lambda\widetilde f_t}$ gives a Gaussian
$R_{t+1}=\N(a_{t+1},C_{t+1})$ with
\begin{align}
 A_{t+1}^+
 &=A_t+\lambda\gamma g_tg_t^\top,
 \label{eq:precision-update}\\
 b_{t+1}^+
 &=A_tm_t-\lambda g_t
 +\lambda\gamma(g_t^\top w_t)g_t,
 \label{eq:linear-update}\\
 C_{t+1}&=(A_{t+1}^+)^{-1},
 \qquad
 a_{t+1}=C_{t+1}b_{t+1}^+.
 \label{eq:mean-update}
\end{align}
These identities follow by collecting the quadratic and linear terms in the exponent.
Sherman--Morrison and the matrix determinant lemma evaluate the update in $O(d^2)$ arithmetic operations once an eigensystem of $A_t$ is stored.

Let
$C_{t+1}=U\operatorname{diag}(c_1,\ldots,c_d)U^\top$ and define the covariance floor
\begin{equation}\label{eq:floor}
 C_{t+1}^{\flat}
 =U\operatorname{diag}(c_j\vee\tau^2)U^\top.
\end{equation}
Then project the tilted mean in the corresponding Mahalanobis metric:
\begin{equation}\label{eq:mahalanobis-proj}
 m_{t+1}
 \in\arg\min_{m\in\cW}
 (m-a_{t+1})^\top(C_{t+1}^{\flat})^{-1}(m-a_{t+1}),
 \qquad
 \Sigma_{t+1}=C_{t+1}^{\flat}.
\end{equation}
Because the tilt only increases precision, $C_{t+1}\preceq\Sigma_t\preceq D^2I$; flooring raises eigenvalues only to $\tau^2\le D^2$, so \eqref{eq:cov-range} is preserved.

For $u\in\cW$, set
\begin{equation}\label{eq:Q-u}
 Q_u=\N(u,\tau^2I_d).
\end{equation}

\begin{lemma}[Comparator-KL monotonicity of the feasibility correction]\label[lemma]{lem:projection-monotone}
For every $u\in\cW$,
\begin{equation}\label{eq:projection-monotone}
 \KL(Q_u\|\N(m_{t+1},\Sigma_{t+1}))
 \le \KL(Q_u\|R_{t+1}).
\end{equation}
\end{lemma}

\begin{proof}
First keep the mean $a=a_{t+1}$ fixed and replace $C=C_{t+1}$ by
$C^\flat=C_{t+1}^\flat$.  In an eigenbasis of $C$, twice the covariance and mean-dependent part of the KL divergence is a sum of
\[
 h_j(c)=\frac{\tau^2+r_j^2}{c}+\log c,
\]
where $r_j$ is the corresponding coordinate of $u-a$.  If $c\le\tau^2$,
\[
 h_j'(c)=\frac{c-\tau^2-r_j^2}{c^2}\le0.
\]
Thus increasing every eigenvalue below $\tau^2$ to $\tau^2$ cannot increase KL.

Now put $A=(C^\flat)^{-1}$ and let $m=m_{t+1}$ be the metric projection of $a$ onto $\cW$.  Its variational inequality is
\[
 (a-m)^\top A(u-m)\le0
 \qquad (u\in\cW).
\]
Expanding squared norms gives
\[
 \norm{u-m}_A^2
 \le\norm{u-a}_A^2-\norm{a-m}_A^2.
\]
The covariance terms are unchanged during mean projection, so KL again cannot increase.
\end{proof}

\paragraph{Mix loss and one-step potential.}

For the Gaussian $P_t$, define
\[
 m_t^{\rm mix}
 =-\frac1\lambda\log\E_{U\sim P_t}e^{-\lambda\widetilde f_t(U)}.
\]

\begin{lemma}[Gaussian mixability gap]\label[lemma]{lem:mixgap}
The history mean satisfies
\begin{equation}\label{eq:mixgap}
 \widetilde f_t(m_t)\le m_t^{\rm mix}.
\end{equation}
\end{lemma}

\begin{proof}
Under $P_t$, the scalar
$X=\ip{g_t}{U-w_t}$ is Gaussian with mean
$a=\ip{g_t}{m_t-w_t}$ and variance $v=g_t^\top\Sigma_tg_t$.
Since $\lambda=\gamma/2$,
$-\lambda\widetilde f_t(U)=-\lambda X-\lambda^2X^2$.
With $s=2\lambda^2v$, direct Gaussian integration~gives
\begin{equation}\label{eq:gaussian-integral}
 \log\frac{\E e^{-\lambda X-\lambda^2X^2}}
 {e^{-\lambda a-\lambda^2a^2}}
 =-\frac12\log(1+s)
 +\frac{s(1+2\lambda a)^2}{4(1+s)}.
\end{equation}
Both $m_t$ and $w_t$ lie in $\cW$, so
$|\lambda a|\le\gamma GD/2\le1/8$ and
$(1+2\lambda a)^2\le25/16$.  Since
$\log(1+s)\ge s/(1+s)$, the right-hand side of
\eqref{eq:gaussian-integral} is at most
$-7s/[64(1+s)]\le0$.  Therefore
\[
 \E e^{-\lambda\widetilde f_t(U)}
 \le e^{-\lambda\widetilde f_t(m_t)},
\]
which is equivalent to \eqref{eq:mixgap}.
\end{proof}

For a comparator path and one history, write
\begin{equation}\label{eq:F-r}
 Q_t=\N(u_t,\tau^2I_d),
 \qquad
 F_t(u)=\KL(Q_u\|P_t),
 \qquad
 r_t=\widetilde f_t(m_t)-\E_{Q_t}\widetilde f_t.
\end{equation}

We can now combine the mix-loss bound with the feasibility correction. The resulting inequality pays for one round of surrogate excess by a decrease in KL to a fixed comparator; the next step will account for changes in that comparator.

\begin{lemma}[One-step KL inequality]\label[lemma]{lem:one-step}
Every history update satisfies
\begin{equation}\label{eq:one-step}
 \lambda r_t\le F_t(u_t)-F_{t+1}(u_t).
\end{equation}
\end{lemma}

\begin{proof}
The exponential-update identity gives
\[
 m_t^{\rm mix}-\E_{Q_t}\widetilde f_t
 =\frac1\lambda\left[
 \KL(Q_t\|P_t)-\KL(Q_t\|R_{t+1})
 \right].
\]
Use \cref{lem:mixgap} on the left and
\cref{lem:projection-monotone} on the right.
\end{proof}

To telescope the preceding inequality along a moving path, we must compare potentials at consecutive comparator points. Taking the square root makes the Gaussian KL Lipschitz in the comparator mean, with sensitivity controlled by the covariance floor. This is the source of the improved movement coefficient.

\begin{lemma}[Square-root KL movement]\label[lemma]{lem:root-movement}
For every $u,v\in\cW$,
\begin{equation}\label{eq:root-movement}
 \sqrt{F_{t+1}(v)}
 \le\sqrt{F_{t+1}(u)}
 +\frac{\norm{v-u}_2}{\sqrt2\,\tau}.
\end{equation}
At a history's birth time $i$,
\begin{equation}\label{eq:initial-KL}
 F_i(u_i)\le B_\tau,
 \qquad
 B_\tau=\frac12\left[
 1+d(\rho^2-1-\log\rho^2)
 \right],
 \quad \rho=\frac\tau D.
\end{equation}
\end{lemma}

\begin{proof}
For fixed $P_{t+1}=\N(m,\Sigma)$,
\[
 F_{t+1}(u)
 =c+\frac12(u-m)^\top\Sigma^{-1}(u-m),
 \qquad
 c=\frac12\left[
 \tau^2\operatorname{tr}(\Sigma^{-1})-d+
 \log\frac{\det\Sigma}{\tau^{2d}}
 \right]\ge0.
\]
Define the augmented vector
\[
 z(u)=\left(
 \sqrt c,\frac1{\sqrt2}\Sigma^{-1/2}(u-m)
 \right).
\]
Then $\sqrt{F_{t+1}(u)}=\norm{z(u)}_2$.  By the triangle inequality and
$\Sigma\succeq\tau^2I$,
\[
 \sqrt{F_{t+1}(v)}
 \le\sqrt{F_{t+1}(u)}
 +\frac1{\sqrt2}\norm{\Sigma^{-1/2}(v-u)}_2
 \le\sqrt{F_{t+1}(u)}+\frac{\norm{v-u}_2}{\sqrt2\tau}.
\]
At birth, $P_i=\N(w_0,D^2I)$; the Gaussian KL formula and
$\norm{u_i-w_0}\le D$ give \eqref{eq:initial-KL}.
\end{proof}

\subsubsection{Shifted drawdown energy and recursive cancellation}\label{app:drawdown}
The movement inequality in \Cref{lem:root-movement} is useful only while its starting potential is controlled. Negative excess regret can increase that potential, so we first absorb cumulative excess into a shifted energy. The resulting bound applies until a large negative prefix appears. We then charge the extra history pieces exposed by recursion to that negative prefix, closing the regret bound.

\begin{lemma}[Shifted drawdown energy]\label[lemma]{lem:drawdown}
Consider a fresh scale-$\tau$ history on $[i,b]$.  If
\begin{equation}\label{eq:drawdown-condition}
 \sum_{t=i}^q r_t\ge-H
 \qquad\text{for every }q\in[i,b],
\end{equation}
then the cumulative excess satisfies
\begin{equation}\label{eq:drawdown-bound}
 \sum_{t=i}^b r_t
 \le
 \frac1\lambda\left(
 \sqrt{B_\tau+\lambda H}
 +\frac{P_{i:b}}{\sqrt2\tau}
 \right)^2-H,
\end{equation}
where $P_{i:b}=\sum_{t=i+1}^b\norm{u_t-u_{t-1}}_2$.
\end{lemma}

\begin{proof}
Put $S_q=\sum_{t=i}^q r_t$, $S_{i-1}=0$, and
\[
 E_t=F_t(u_t)+\lambda(S_{t-1}+H).
\]
The drawdown assumption makes the added term nonnegative.  By
\eqref{eq:one-step},
\[
 F_{t+1}(u_t)+\lambda(S_t+H)\le E_t.
\]
For $t<b$, use \eqref{eq:root-movement} and the fact that adding the same nonnegative scalar under both square roots preserves the movement inequality:
\[
 \sqrt{E_{t+1}}
 \le\sqrt{E_t}
 +\frac{\norm{u_{t+1}-u_t}_2}{\sqrt2\tau}.
\]
Iteration and \eqref{eq:initial-KL} give
\[
 \sqrt{E_b}
 \le\sqrt{B_\tau+\lambda H}
 +\frac{P_{i:b}}{\sqrt2\tau}.
\]
Applying \eqref{eq:one-step} at $t=b$ yields
$F_{b+1}(u_b)+\lambda(S_b+H)\le E_b$.
Drop the terminal KL term.
\end{proof}

Fix one scale and abbreviate
\begin{equation}\label{eq:thresholds}
 L=\lceil\log_2(T+1)\rceil,
 \quad
 A_\tau=a+\frac{B_\tau}{\lambda},
 \quad
 H_\tau=128LA_\tau,
 \quad
 C_\tau=B_\tau+\lambda H_\tau,
 \quad
 V_\tau=\frac{\tau\sqrt{C_\tau}}{512L}.
\end{equation}
For an interval $I$, define
\begin{equation}\label{eq:Rtau}
 \mathcal R_\tau(I)
 =\sum_{t\in I}
 [\widetilde f_t(w_t)-\E_{Q_t}\widetilde f_t].
\end{equation}

The shifted-energy bound concerns one fresh history and assumes that its cumulative excess never falls below a threshold. The next lemma removes that condition: whenever it fails, the negative prefix pays for the additional histories introduced by recursion. This converts a conditional history bound into an unconditional interval bound for the master.

\begin{lemma}[Low-variation interval bound]\label[lemma]{lem:low-variation}
If $P_I\le V_\tau$, then
\begin{equation}\label{eq:low-var}
 \mathcal R_\tau(I)\le\frac{H_\tau}{64}.
\end{equation}
\end{lemma}

\begin{proof}
Recursively decompose $I=[r,s]$.  By \cref{lem:schedule}, choose
$i\in[r,\lfloor(r+s)/2\rfloor]$ whose history remains active through $s$.
Recurse on $[r,i-1]$.  On $[i,s]$, define
\[
 S_q=\sum_{t=i}^q
 [\widetilde f_t(x_t^{i,k})-\E_{Q_t}\widetilde f_t].
\]
If $S_q\ge-H_\tau$ for every $q$, call the node good and terminate its right branch.
Otherwise let $q$ be the first crossing of $-H_\tau$, charge $[i,q]$ as a bad node, and recurse on $[q+1,s]$.
The charged pieces and recursive subintervals are disjoint and partition $I$.

For a good node, \cref{lem:meta,lem:drawdown} and
$P_{i:s}\le P_I\le V_\tau$ imply
\begin{align*}
 \mathcal R_\tau([i,s])
 &\le a+\frac1\lambda
 \left(\sqrt{C_\tau}+\frac{V_\tau}{\sqrt2\tau}\right)^2-H_\tau\\
 &\le\frac{H_\tau}{64L}.
\end{align*}
Indeed, with
$\varepsilon=V_\tau/(\sqrt2\tau\sqrt{C_\tau})
=1/(512\sqrt2L)$, the first line equals
\[
 A_\tau+(2\varepsilon+\varepsilon^2)\frac{C_\tau}{\lambda}.
\]
Since $C_\tau/\lambda=B_\tau/\lambda+H_\tau
\le(1+128L)A_\tau$, this is at most
$2A_\tau=H_\tau/(64L)$.

For a bad node, \cref{lem:meta} and the crossing definition give the charge
\[
 a-H_\tau\le-\frac{127}{128}H_\tau.
\]
Remove all bad nodes from the recursion tree.  Every remaining connected component consists only of left-child edges, because a right child is created only by a bad node.  Along a left edge the interval length at least halves, so each component has at most $L$ good nodes.  Removing $N_{\rm bad}$ binary-tree nodes creates at most $2N_{\rm bad}+1$ components.  Hence
\[
 N_{\rm good}\le L(2N_{\rm bad}+1).
\]
The total charge is at most
\[
 \frac{H_\tau}{64L}L(2N_{\rm bad}+1)
 -\frac{127H_\tau}{128}N_{\rm bad}
 \le\frac{H_\tau}{64}.
\]
\end{proof}

\subsubsection{Path partition, scale optimization, and endpoint regimes}\label{app:scales}
\Cref{lem:low-variation} gives an interval cost in terms of its internal path length. Partitioning a general path at a scale-dependent variation threshold converts this local statement into a bound linear in $P_T/\tau$. We optimize the scale and include both endpoints.

\begin{lemma}[Variation partition]\label[lemma]{lem:partition}
For every $V>0$, every comparator path can be partitioned into consecutive intervals
$I_1,\ldots,I_M$ satisfying
\[
 P_{I_j}\le V,
 \qquad
 M\le1+\frac{P_T}{V}.
\]
\end{lemma}

\begin{proof}
Greedily extend the current interval until the next jump would make its internal variation exceed $V$, then cut immediately before that jump.  For every completed interval, its internal variation plus the triggering boundary jump exceeds $V$.  These collections of jumps are disjoint, so $(M-1)V<P_T$.
\end{proof}

Combining the variation partition with the low-variation interval bound gives a guarantee for an arbitrary comparator path. The remaining term is the error from Gaussian smoothing. The next lemma displays these two contributions before optimizing the scale.

\begin{lemma}[Explicit fixed-scale guarantee]\label[lemma]{thm:fixed-scale}
For every $\tau=\tau_k$, simultaneously for every comparator path,
\begin{equation}\label{eq:fixed-scale}
 \Reg_T(u_{1:T})
 \le
 \frac{H_\tau}{64}\left(1+\frac{P_T}{V_\tau}\right)
 +\frac{\gamma G^2T\tau^2}{2}.
\end{equation}
Consequently, for a universal $C$,
\begin{equation}\label{eq:fixed-simplified}
 \Reg_T(u_{1:T})
 \le C\left[
 \frac{L\overline B_\tau}{\gamma}
 +\frac{L^{3/2}\sqrt{\overline B_\tau}}{\gamma\tau}P_T
 +\gamma G^2T\tau^2
 \right],
\end{equation}
where the combined initial budget is
\begin{equation}\label{eq:Bbar}
 \overline B_\tau=B_\tau+\log((T+1)(K+1)).
\end{equation}
\end{lemma}

\begin{proof}
Apply \cref{lem:partition} with $V=V_\tau$ and
\cref{lem:low-variation} on each bin.  For
$Q_t=\N(u_t,\tau^2I)$,
\[
 \E_{Q_t}\widetilde f_t-
 \widetilde f_t(u_t)
 =\frac\gamma2\tau^2\norm{g_t}_2^2
 \le\frac\gamma2G^2\tau^2.
\]
Together with \eqref{eq:domination}, this proves \eqref{eq:fixed-scale}.
Substitute \eqref{eq:thresholds}, use
$a+B_\tau/\lambda=\overline B_\tau/\lambda$, and
$\lambda=\gamma/2$ to obtain \eqref{eq:fixed-simplified}.
\end{proof}

\Cref{thm:fixed-scale} holds simultaneously on the dyadic grid, so we can now compare with the best scale in hindsight. The formal upper bound must also cover paths for which the unconstrained balancing scale lies outside the grid range. The next lemma handles both endpoints and the discretization loss.

\begin{lemma}[Scale endpoints]\label[lemma]{lem:endpoints}
Define the coefficients of the scale-dependent bound by
\[
 \overline B=C_0(d+1)\Lambda_T,
 \qquad
 A=\frac{L\overline B}{\gamma},
 \qquad
 B=\frac{L^{3/2}\sqrt{\overline B}}{\gamma},
 \qquad
 C_2=\gamma G^2T,
\]
where $C_0$ is large enough that
$\overline B_\tau\le\overline B$ for every
$\tau\in[D/T,D]$.  Put $\tau_*=(BP_T/C_2)^{1/3}$ and
$P_0=C_2D^3/B$.  Then some dyadic scale satisfies
\begin{equation}\label{eq:endpoint-bound}
 \Reg_T(u_{1:T})
 \le C\left[
 A+\min_{\tau\in[D/T,D]}
 \left(\frac{BP_T}{\tau}+C_2\tau^2\right)
 \right]\wedge GDT.
\end{equation}
Moreover, whenever $P_T\le P_0$,
\begin{equation}\label{eq:interior-endpoint}
 \Reg_T(u_{1:T})
 \le C\left[A+(B^2C_2P_T^2)^{1/3}\right].
\end{equation}
\end{lemma}

\begin{proof}
If $\tau_*\in[D/T,D]$, choose a dyadic scale within a factor two and optimize
$BP_T/\tau+C_2\tau^2$.

If $\tau_*<D/T$, choose $\tau_K\in[D/T,2D/T)$. Since
$BP_T/C_2<D^3/T^3$, we obtain
\[
 \frac{BP_T}{\tau_K}+C_2\tau_K^2
 \le5\gamma G^2\frac{D^2}{T}.
\]
Because $\gamma GD\le1/4$,
$\gamma G^2D^2/T\le1/(16\gamma T)\le A$ after adjusting the universal constant.  Thus the static term absorbs the lower endpoint and \eqref{eq:interior-endpoint} remains valid.

If $\tau_*>D$, choosing $D$ is within a universal factor of the clipped minimum in \eqref{eq:endpoint-bound}; independently, Lipschitzness gives the cap $GDT$.  This proves the all-path clipped statement.  Finally, $P_T\le P_0$ is exactly the condition $\tau_*\le D$, so the preceding two cases establish \eqref{eq:interior-endpoint}.
\end{proof}

\subsubsection{The projected-gradient bound}\label{app:dogd}

The curved learner is most effective for slowly moving comparators. We combine it with a grid of projected-gradient learners to control larger path budgets, using the classical dynamic projected-gradient argument~\cite{Zinkevich2003,ZhangEtAl2018}.

\begin{lemma}[Dynamic projected gradient descent]\label[lemma]{lem:dogd}
Let $z_{t+1}=\Pi_{\cW}(z_t-\alpha\nabla f_t(z_t))$ for a fixed $\alpha>0$, where $\Pi_{\cW}$ is Euclidean projection.  Then every comparator path satisfies
\begin{equation}\label{eq:dogd}
 \sum_{t=1}^T[f_t(z_t)-f_t(u_t)]
 \le\frac{D^2+2DP_T(u)}{2\alpha}
 +\frac{\alpha G^2T}{2}.
\end{equation}
\end{lemma}

\begin{proof}
Convexity and projection nonexpansiveness give
\[
 2\alpha[f_t(z_t)-f_t(u_t)]
 \le\norm{z_t-u_t}^2-\norm{z_{t+1}-u_t}^2+\alpha^2G^2.
\]
The change in comparator also satisfies
\[
 \norm{z_{t+1}-u_t}^2-\norm{z_{t+1}-u_{t+1}}^2
 \ge-2D\norm{u_{t+1}-u_t},
\]
because all three points lie in a set of diameter $D$.  Sum and telescope, using
$\norm{z_1-u_1}\le D$.
\end{proof}

\subsubsection{Numerical model and robustness}\label{app:numerical-lemmas}

We now separate exact arithmetic from an actual bit model.

\begin{assumption}[Polynomial-bit online oracle model]\label[assumption]{ass:bit-model}
All scalar parameters and the initial point have rational encodings of at most $L_0$ bits. On a rational query point, the online loss oracle returns exact rational loss and gradient values with at most $L_0$ bits per coordinate. Given rational $A\succ0$ and $a$ with at most $B$ bits per coordinate, and a precision request $p$, the projection oracle returns a feasible rational point $m\in\cW$ whose Mahalanobis objective is within $2^{-p}$ of optimum. Its coordinates have at most $\operatorname{poly}(d,L_0,B,p)$ bits, and its cost is
$\mathsf C_{\rm proj}^{\rm bit}(d,L_0,B,p)$ bit operations. Rational convex combinations of feasible points can be formed exactly in polynomial bit complexity. For the near-linear-in-$T$ conclusion, $\mathsf C_{\rm proj}^{\rm bit}$ is polynomial in its displayed arguments.
\end{assumption}

This model specifies both feasibility and the size of the returned encodings. We analyze the implemented predictions directly, with exact oracle replies at those predictions.

\begin{lemma}[Robustness to inexact history and master updates]\label[lemma]{lem:robust}
Suppose every selected Gaussian update satisfies, uniformly over $u\in\cW$,
\begin{equation}\label{eq:approx-kl}
 \lambda r_t
 \le F_t(u_t)-F_{t+1}(u_t)+\varepsilon_t,
\end{equation}
and the implemented master satisfies, for every surviving-history prefix $[i,q]$,
\begin{equation}\label{eq:approx-master}
 \sum_{t=i}^q[\widetilde f_t(w_t)-\widetilde f_t(x_t^{i,k})]
 \le a+\lambda^{-1}\sum_{t=i}^q\zeta_t.
\end{equation}
Here $\varepsilon_t,\zeta_t\ge0$ bound the errors uniformly over active histories, including weight changes at birth and pruning.
Then all conclusions of \cref{thm:curved-detailed} hold with the additive term
\begin{equation}\label{eq:numerical-additive}
 \frac1\lambda\sum_{t=1}^T(\varepsilon_t+\zeta_t).
\end{equation}
\end{lemma}

\begin{proof}
Replace $r_t$ by
$\bar r_t=r_t-\varepsilon_t/\lambda$.  Then the exact one-step inequality holds for $\bar r_t$, so the shifted-energy and recursion arguments apply unchanged.  Similarly, \eqref{eq:approx-master} changes every selected prefix overhead by at most the sum of its $\zeta_t/\lambda$.  The charged and recursively treated pieces are disjoint, so every round's numerical error is counted at most once.  Restoring $r_t$ gives \eqref{eq:numerical-additive}.
\end{proof}

\subsection{Full version and proof of \texorpdfstring{\Cref{thm:curved}}{Theorem \ref{thm:curved}}}\label{app:full-statements}
The next statement makes the dependence on curvature, diameter, gradient bound, and smoothing scale explicit. It also records the clipped-scale guarantee for every comparator path.

\begin{theorem}[Full version of \Cref{thm:curved}]\label[theorem]{thm:curved-detailed}
There are universal constants $C,c>0$ such that, under \Cref{ass:model,ass:exact-oracle}, there is a proper full-information algorithm, independent of $P_T(u)$, such that every comparator path satisfies
\begin{align}
 \Reg_T(u_{1:T})
 \le C\biggl\{
 &\frac{(d+1)\Lambda_T^2}{\gamma}
 +\min_{\tau\in[D/T,D]}
 \left[
 \frac{\sqrt{d+1}\,\Lambda_T^2}{\gamma\tau}P_T(u)
 +\gamma G^2T\tau^2
 \right]
 \biggr\}\wedge GDT.
 \label{eq:curved-main}
\end{align}
In particular, if
\begin{equation}\label{eq:curved-regime}
 P_T(u)\le
 c\,\frac{\gamma^2G^2TD^3}{\sqrt{d+1}\,\Lambda_T^2},
\end{equation}
then the bound simplifies to
\begin{equation}\label{eq:curved-specialized}
 \Reg_T(u_{1:T})
 \le C\left[
 \frac{(d+1)\Lambda_T^2}{\gamma}
 +\gamma^{-1/3}G^{2/3}(d+1)^{1/3}\Lambda_T^{4/3}
 T^{1/3}P_T(u)^{2/3}
 \right].
\end{equation}
The curved algorithm uses one gradient query per round at its actual prediction and has
\begin{align}
 \operatorname{Time}_{\rm arith}
 &\le C T\log^2(2T)\,[d^3+\mathsf C_{\rm proj}(d)],
 \label{eq:arith-time}\\
 \operatorname{Memory}
 &\le C d^2\log^2(2T).
 \label{eq:arith-memory}
\end{align}
If a primitive implements the complete rank-one precision update, covariance flooring, and reconstruction in the original coordinates in $O(d^2)$ operations, the $d^3$ term in \eqref{eq:arith-time} can be replaced by $d^2$.
\end{theorem}

\begin{proof}
\Cref{thm:fixed-scale} and \cref{lem:endpoints} give
\[
 A\le C\frac{(d+1)\Lambda_T^2}{\gamma},
 \qquad
 B\le C\frac{\sqrt{d+1}\,\Lambda_T^2}{\gamma}.
\]
This proves the clipped bound \eqref{eq:curved-main}.  Under
\eqref{eq:curved-regime}, $P_T\le P_0$ up to a universal adjustment, and
\[
 (B^2C_2P_T^2)^{1/3}
 \le C\gamma^{-1/3}G^{2/3}(d+1)^{1/3}
 \Lambda_T^{4/3}T^{1/3}P_T^{2/3},
\]
which proves \eqref{eq:curved-specialized}.

By \cref{lem:schedule}, there are $O(\log T)$ active starts and
$O(\log T)$ scales, hence $O(\log^2T)$ Gaussian states.  A standard symmetric eigendecomposition, covariance clipping, and Mahalanobis projection cost
$O(d^3+\mathsf C_{\rm proj})$ per state in the dense spectral model. The conditional acceleration substitutes an $O(d^2)$ primitive for this complete spectral update, including reconstruction of the floored covariance in the original coordinates. Interlacing limits the number of precision eigenvalues crossing the cap to one; the primitive must also account for the change of eigenbasis. This proves \eqref{eq:arith-time}--\eqref{eq:arith-memory} under the respective computational models.
For fixed $D,G,\eta$, the constant $\gamma$ is fixed and $d+1\le2d$. Absorbing these regularity-dependent constants gives \Cref{thm:curved}, with the displayed logarithmic factors unchanged.
\end{proof}

\subsection{Full version and proof of \texorpdfstring{\Cref{cor:frontier}}{Corollary \ref{cor:frontier}}}\label{app:frontier}
We now combine the curved learner with the projected-gradient bound to cover all path budgets.

Run a grid
\[
 \alpha_j=2^j\frac{D}{G\sqrt T},
 \qquad
 j=0,\ldots,\left\lceil\log_2\sqrt{2T}\right\rceil,
\]
plus the curved algorithm of \cref{thm:curved}.  Use an outer exponential-weights master on the \emph{original} losses $f_t$, at learning rate $\eta$.  Every base prediction is in $\cW$, so their weighted mean is proper; exp-concavity gives outer meta-regret at most
$\eta^{-1}\log O(\log T)$ to every base.
The curved base is run counterfactually at its own prediction and queries one gradient there.  Each gradient-descent base queries a gradient at its own prediction; this is why \cref{cor:frontier} states $O(\log T)$ gradient evaluations per round.

\begin{corollary}[Full version of \Cref{cor:frontier}]\label[corollary]{cor:frontier-detailed}
Under \Cref{ass:model,ass:exact-oracle}, a proper full-information learner, independent of $P_T(u)$, satisfies
\begin{align}\label{eq:full-frontier-detailed}
 \Reg_T(u_{1:T})\le C\biggl[&\frac{(d+1)\Lambda_T^2}{\gamma}
 +\min\biggl\{
 \gamma^{-1/3}G^{2/3}(d+1)^{1/3}\Lambda_T^{4/3}T^{1/3}P_T(u)^{2/3},\nonumber\\
 &G\sqrt{T(D^2+DP_T(u))},\ GDT\biggr\}\biggr]
\end{align}
for a universal constant $C>0$. It uses $O(\log T)$ loss evaluations, gradient evaluations, and Euclidean projection calls per round in addition to the curved branch's internal work.
\end{corollary}

\begin{proof}
For the optimal fixed step in \eqref{eq:dogd},
\[
 \alpha_*=\sqrt{\frac{D^2+2DP_T}{G^2T}},
\]
which lies in the range covered by the grid up to a factor two.  Therefore one grid expert has regret
\[
 U_T=O\!\left(G\sqrt{T(D^2+DP_T)}\right).
\]
The outer exp-concave master pays only $O(\eta^{-1}\log\log T)$ to the better base.

It remains to justify writing the unconstrained curved expression inside the minimum even when the curved smoothing scale would exceed $D$.  With the notation of \cref{lem:endpoints}, let
$Q_T=(B^2C_2P_T^2)^{1/3}$ and $P_0=C_2D^3/B$.  For $P_T\le P_0$, \cref{thm:curved-detailed} supplies $A+Q_T$.  For $P_T\ge P_0$, direct algebra gives
\[
 G\sqrt{TDP_T}\le Q_T,
 \qquad
 GD\sqrt T\le\sqrt{AQ_T}\le\frac{A+Q_T}{2},
\]
because $\gamma B\ge1$ after the universal constants are fixed.  Hence
$U_T\le C(A+Q_T)$ also outside the curved-scale range.  Thus the better of the curved and OGD bases is bounded by the first two branches in \eqref{eq:full-frontier-detailed}. The $GDT$ cap follows from Lipschitzness and the diameter bound. Substituting the definitions of $A,B,C_2$ gives the displayed logarithmic factors; the outer aggregation cost is absorbed by $(d+1)\Lambda_T^2/\gamma$.

With $D,G,\eta$ fixed, absorbing regularity-dependent constants gives \Cref{cor:frontier}. For the fixed-parameter form \eqref{eq:intro-frontier}, the replacement of $\sqrt{T(1+P_T)}$ by $\sqrt{TP_T}$ is valid inside the displayed minimum, with the additive $d$ retained. If $P_T\ge1$, the two square-root expressions differ by at most $\sqrt2$. If $P_T<1$ and $T\ge d^2$, the curved term is at most $\sqrt{TP_T}$. If $P_T<1$ and $T<d^2$, the extra $\sqrt T$ is at most $d$. This also justifies the abstract's upper bound for all horizons.
\end{proof}

\subsection{Finite precision and Turing complexity}\label{sec:finite-precision}
Under the rational-oracle model in \Cref{ass:bit-model}, the local error control of \Cref{lem:robust} yields the following implementation guarantee.

\begin{theorem}[Polynomial-bit implementation]\label[theorem]{thm:bit}
Under \Cref{ass:model,ass:bit-model}, for every $\delta\in(0,1)$ there is a Turing implementation of the curved algorithm whose regret is the right-hand side of \eqref{eq:curved-main} plus $\delta$. In addition to one loss/gradient oracle call per round, its total bit complexity is
\begin{equation}\label{eq:bit-complexity}
 \widetilde O\!\left(
 T\log^2(2T)
 \left[
 \operatorname{poly}\!\left(d,L_0,\log T,\log\frac1\delta,
 \log(1+1/\gamma),\log(1+D+G)\right)
 +\mathsf C_{\rm proj}^{\rm bit}(d,L_0,B,p)
 \right]
 \right),
\end{equation}
where the oracle precision parameters satisfy
\[
 B,p=\operatorname{poly}\!\left(
 d,L_0,\log(T/\delta),\log(1+1/\gamma),\log(1+D+G)
 \right).
\]
The horizon dependence is near-linear under the polynomial-bit oracle model above.
\end{theorem}

\begin{proof}
We control local update errors at the implemented state, so their contribution is additive through \Cref{lem:robust}. Write $N=O(\log^2(2T))$ for a uniform bound on the number of active states and set the target local errors to $\varepsilon_t,\zeta_t\le\lambda\delta/(2T)$.

\emph{Uniform conditioning.}
Work in coordinates translated by the rational initial point $w_0$. Every feasible mean then has norm at most $D$. The stored covariances satisfy $\tau^2I\preceq\Sigma\preceq D^2I$, with $\tau\ge D/T$. Before flooring, the precision matrix is bounded between $D^{-2}I$ and $(\tau^{-2}+\lambda\gamma G^2)I$. The tilt mean is obtained by solving this positive-definite system with a bounded right-hand side. Consequently, all relevant magnitudes, inverse eigenvalues, and derivatives of the Gaussian KL formula have bounds of the form $2^{\operatorname{poly}(L_0)}\operatorname{poly}(d,T,1+1/\gamma,1+D+G)$. This includes inverse powers of $D$, since its positive rational encoding has at most $L_0$ bits. Thus polynomially many guard bits in the stated parameters suffice for any prescribed local KL error.

\emph{A rational covariance floor.}
The scale $\tau=D$ has the exact stored covariance $D^2I$. For the other dyadic scales, $\tau\le D/2$. Compute a rational approximation to the spectral clipping of the tilted covariance into $[\tau^2,D^2]$, first moving the two clipping thresholds inward by a guard margin $h$ and then approximating the matrix to operator-norm error at most $h/2$. The stored rational covariance then remains in the certified interval $[\tau^2,D^2]$ and differs from the exact floor by $O(h)$.

This operation has polynomial bit complexity without an eigenvalue separation assumption. One explicit construction replaces the scalar positive part $x_+$ by $(x+\sqrt{x^2+\epsilon^2})/2$, with uniform error at most $\epsilon/2$, and applies this formula successively at the lower and upper thresholds. For a symmetric matrix $M$, the required positive-definite square root is that of $M^2+\epsilon^2I$. Starting from a scalar upper bound times the identity, Newton's iteration $X\leftarrow(X+X^{-1}(M^2+\epsilon^2I))/2$ reduces to the scalar square-root iteration in the eigenbasis of $M$. Its iteration count is polynomial in the logarithmic spectral bounds and the requested precision. Rational rounding with guard bits gives the required operator-norm error; entrywise error $e$ contributes at most $de$ in operator norm. Choosing the smoothing, iteration, and rounding errors below fixed fractions of $h$ proves the claimed polynomial construction. The uniform KL bounds above turn its $O(h)$ covariance error into the desired local KL error.

\emph{Feasible mean updates.}
Round the tilt mean and the positive-definite projection metric to the working precision before calling the projection oracle. Their input encodings therefore have at most $B$ bits per coordinate, independent of earlier oracle-output lengths. Perturbing these inputs changes the quadratic objective uniformly on $\cW$ by a quantity controlled by the same magnitude bounds. Hence a sufficiently accurate oracle output has objective gap at most $\xi$ for the exact floored metric $A_*=(C^\flat)^{-1}$. If $m^*$ is its exact projection and $\widehat m$ is the feasible returned point, the projection variational inequality gives
\[
 \|\widehat m-m^*\|_{A_*}^2\le\xi,
 \qquad
 \|u-\widehat m\|_{A_*}^2-\|u-m^*\|_{A_*}^2
 \le2(D/\tau)\sqrt\xi+\xi.
\]
The covariance and mean errors therefore yield \eqref{eq:approx-kl} uniformly over $u\in\cW$. The oracle's output-length bound makes each stored feasible mean polynomial in the working precision. At the next update, its derived projection inputs are rounded again to $B$ bits, so encoding lengths do not grow with the number of rounds.

\emph{Master weights.}
Absolute rounding alone can erase a small expert weight. We instead round each normalized probability upward. For an exact normalized vector $v$ with at most $N$ entries, compute rational upper approximations $b_j$ such that
\[
 v_j+h\le b_j\le v_j+3h,
 \qquad q_j=\frac{b_j}{\sum_\ell b_\ell}.
\]
Certified absolute approximations to the exponentials suffice to construct these bounds. They imply
\[
 q_j\ge\frac{v_j}{1+3Nh},
 \qquad q_j\ge\frac{h}{1+3Nh}.
\]
Apply this operation after the exponential update and after birth, deletion, and normalization. Each operation decreases an individual weight by at most a factor $1+3Nh$ relative to its ideal value at that implemented state. In the telescoping proof of \Cref{lem:meta}, this adds at most $C Nh$ per round to the logarithmic overhead, including the birth weight. It proves \eqref{eq:approx-master} with $\zeta_t\le C Nh$. All stored probabilities have polynomial encoding length and a known positive lower bound. Since $|\lambda\widetilde f_t(x)|$ is uniformly bounded on $\cW$, the required normalized exponential probabilities can be evaluated to this absolute accuracy in polynomial bit complexity.

The prediction is the exact rational convex combination of feasible means and the stored rational weights. The loss oracle is queried at this actual prediction, so surrogate domination and the Gaussian mix-gap argument apply directly to the implemented sequence. Choose $h$ and the projection objective accuracy so that $\varepsilon_t,\zeta_t\le\lambda\delta/(2T)$. The preceding bounds require only the displayed polynomial values of $B,p$. \Cref{lem:robust} then bounds the additional regret by $\delta$. There are $O(\log^2(2T))$ states per round, giving \eqref{eq:bit-complexity}.
\end{proof}

\begin{remark}[Anytime implementation]
A doubling schedule removes prior knowledge of $T$. For the curved learner, fix a comparison scale $\tau\in[D/T,D]$ and in epoch $j$ compare with a dyadic scale within a factor two of $\max\{\tau,D/H_j\}$, where $H_j$ is that epoch's planned doubling horizon. Its movement term is bounded by a constant times $P_j/\tau$. The additional smoothing cost caused by the lower endpoint is at most $C\gamma G^2D^2/H_j$, and the geometric sum of these costs is absorbed by the static term. Summing the fixed-scale bounds and then optimizing this common $\tau$ recovers the clipped-scale guarantee with additional logarithmic factors.

Restarting the full-frontier ensemble likewise gives the full frontier with logarithmic overhead. If epoch $j$ has planned horizon $H_j$ and internal path length $P_j$, then H\"older's inequality gives
\[
 \sum_jH_j^{1/3}P_j^{2/3}
 \le\left(\sum_jH_j\right)^{1/3}
 \left(\sum_jP_j\right)^{2/3}.
\]
The planned horizons satisfy $\sum_jH_j\le4T$, including the last partially completed epoch. The square-root branch therefore sums to at most $2\sqrt{T(J+\sum_jP_j)}$ over $J=O(\log(2T))$ epochs, by Cauchy--Schwarz, and the actual linear caps sum to $T$. For a total numerical regret budget $\delta$, allocate $\delta_j=\delta/2^{j+1}$ to epoch $j$; then $\sum_j\delta_j\le\delta$ and its precision still grows only polynomially in $\log H_j$ and $\log(1/\delta)$.
\end{remark}

\section{Proofs for Section~\ref{sec:lower-main}}\label{sec:lower}
The lower construction allocates uncertainty across time blocks and coordinates. Each coordinate carries a hidden sign whose noisy labels reveal only a constant amount of information within one block. The comparator follows the blockwise conditional mean. The norm constraint limits the total signal across coordinates, while the path budget limits how often that signal can change. Choosing the number of active coordinates recovers the curved, square-root, and saturation regimes in \Cref{thm:lower}.

\subsection{Useful lemmas}\label{app:lower-lemmas}
The first lemma bounds sequential estimation of one hidden sign. The next two convert that risk into movement-dependent and stationary lower bounds.

\begin{lemma}[Sequential two-point risk]\label[lemma]{lem:two-point}
Let $S$ be uniform on $\{-1,+1\}$ and, conditional on $S$, let
$Y_1,\ldots,Y_n\in\{-1,+1\}$ be independent with
$\Prb(Y_r=1\mid S)=(1+S\theta)/2$, where $0<\theta\le1/2$.
If $n\theta^2\le1/32$, then every sequential prediction $X_r\in[-1,1]$ satisfies
\begin{equation}\label{eq:two-point-risk}
 \sum_{r=1}^n\E(X_r-S\theta)^2\ge c n\theta^2
\end{equation}
for a universal $c>0$, where $X_r$ is measurable with respect to $Y_1,\ldots,Y_{r-1}$ and auxiliary randomness independent of $(S,Y_{1:n})$. Independent observations from other hidden coordinates may be included in this auxiliary randomness.
\end{lemma}

\begin{proof}
The one-sample KL divergence between the two sign models is
\[
 \theta\log\frac{1+\theta}{1-\theta}\le4\theta^2.
\]
Before round $r$, Pinsker's inequality bounds the total variation between the two histories by
$\sqrt{2(r-1)}\,\theta\le1/4$.  Thus every sign test errs with probability at least $3/8$.  The test $\operatorname{sign}(X_r)$ has
$|X_r-S\theta|\ge\theta$ on a sign error.  Sum over $r$.
\end{proof}

The two-point lemma gives a constant amount of prediction error while one hidden sign remains uncertain. We replicate that experiment over coordinates and time blocks. The effective dimension $m$ controls how much independent uncertainty can be packed into the fixed Euclidean ball without exceeding the path budget.

\begin{lemma}[Block-hypercube lower bound in effective dimension $m$]\label[lemma]{lem:effective-lower}
There are universal $c,C>0$ such that the following holds.  Let $1\le m\le d$ and suppose
\begin{equation}\label{eq:effective-regime}
 C\frac m{\sqrt T}\le P\le c\frac T{m^2}.
\end{equation}
Then the loss class \eqref{eq:hard-loss} admits a distribution and a comparator path of length at most $P$ for which every learner suffers
\begin{equation}\label{eq:effective-rate}
 \E\Reg_T\ge c m^{1/3}T^{1/3}P^{2/3}.
\end{equation}
\end{lemma}

\begin{proof}
Use only the first $m$ coordinates.  Choose
\[
 K=\left\lceil C_1\left(\frac{TP^2}{m^2}\right)^{1/3}\right\rceil,
 \qquad
 n=\left\lfloor\frac{T}{Km}\right\rfloor,
 \qquad
 a=\frac{P}{4K},
 \qquad
 \theta=\frac a{\sqrt m}.
\]
For suitable constants in \eqref{eq:effective-regime},
$2\le K\le T/(2m)$ and $a\le1/2$.
Partition $Knm$ rounds into $K$ blocks.  In block $k$, draw independent signs
$S_{k,1},\ldots,S_{k,m}$ and set
\[
 u^{(k)}=\frac a{\sqrt m}(S_{k,1},\ldots,S_{k,m},0,\ldots,0).
\]
This lies in $\cW_d$, and the total path is at most
$2a(K-1)\le P$.
Within each block, cycle through the $m$ active coordinates $n$ times.  When coordinate $j$ is queried, draw $Y\in\{-1,+1\}$ with mean
$\theta S_{k,j}$ and reveal $f_{j,Y}$.  Keep the last block comparator fixed on any leftover rounds and continue the same conditional-mean square-loss experiment.  Those leftover rounds have nonnegative conditional expected excess, so they may be discarded in the lower-bound calculation.

Conditional on the signs and the past observations $\mathcal H_{t-1}$, every prediction $w_j$ satisfies
\begin{equation}\label{eq:square-risk}
 \E\left[f_{j,Y}(w)-f_{j,Y}(u^{(k)})\mid S_{k,j},\mathcal H_{t-1}\right]
 =\frac12(w_j-\theta S_{k,j})^2.
\end{equation}
The information available for each hidden sign is controlled by
\[
 n\theta^2
 \le\frac{TP^2}{16K^3m^2},
\]
which is at most $1/32$ for a sufficiently large $C_1$.  Apply
\cref{lem:two-point} independently to every block-coordinate pair.  Since
$Knm\ge T/2$,
\[
 \E\Reg_T\ge cT\theta^2
 =c\frac{TP^2}{mK^2}
 \ge c m^{1/3}T^{1/3}P^{2/3}.
\]
The calculation holds after conditioning on any learner randomness. Averaging over the generated loss sequences and comparator paths therefore implies that, for every randomized learner, at least one admissible deterministic pair attains the same expected lower bound.
\end{proof}

The block construction captures the movement-dependent part of regret, but its stated regime excludes very small path budgets. A stationary comparator supplies the missing initialization cost. The next lemma provides this additive dimension term without using any movement.

\begin{lemma}[Static lower bound]\label[lemma]{lem:static-lower}
If $T\ge c_0d^2$, then the same class has static minimax regret $\Omega(d)$.
\end{lemma}

\begin{proof}
Draw independent fixed signs $S_1,\ldots,S_d$ and set the fixed comparator
$u_j=S_j/(8\sqrt d)$. Query each coordinate for $n=d$ initial rounds with labels of mean $S_j/(8\sqrt d)$.  On all remaining rounds continue to use losses of the form \eqref{eq:hard-loss} with the same conditional means.  Their conditional expected excess over the fixed comparator is nonnegative, so it is enough to lower-bound the initial observations.  The product $n\theta^2=1/64$ satisfies \Cref{lem:two-point}, and \cref{lem:two-point,eq:square-risk} yield constant regret per coordinate and hence $\Omega(d)$ total regret.  Averaging over the signs and observations yields an admissible deterministic loss sequence and fixed comparator with at least this regret in expectation over learner randomness.
\end{proof}

\subsection{Proof of \texorpdfstring{\Cref{thm:lower}}{Theorem \ref{thm:lower}}}\label{app:lower-theorem}
The following restatement collects the normalized lower frontier. Its proof verifies the loss class and then combines the preceding lemmas across all path budgets.

\begin{theorem}[Restatement of \Cref{thm:lower}]\label[theorem]{thm:lower-restated}
There are universal constants $c,c_0>0$ such that the following holds.  Let
\[
 \cW_d=\{w\in\R^d:\norm w_2\le1/2\}
\]
and let the adversary use only the losses
\begin{equation}\label{eq:hard-loss-restated}
 f_{j,y}(w)=\frac12(w_j-y)^2,
 \qquad j\in[d],\quad y\in\{-1,+1\}.
\end{equation}
This class has diameter $D=1$, gradient bound $G\le3/2$, and exp-concavity coefficient at least $4/9$, uniformly in $d$.  If $T\ge c_0d^2$, then for every $P\ge0$,
\begin{equation}\label{eq:lower-frontier-restated}
 \inf_{\mathsf A}\sup_{\substack{f_{1:T}\text{ of the form }\eqref{eq:hard-loss}\\
 u_{1:T}\in\cW_d^T,\ P_T(u)\le P}}
 \E_{\mathsf A}\Reg_T(u_{1:T})
 \ge
 c\left[
 d+\min\left\{
 d^{1/3}T^{1/3}P^{2/3},\sqrt{TP},T
 \right\}
 \right].
\end{equation}
\end{theorem}

\begin{proof}
We first verify the stated regularity. For \eqref{eq:hard-loss},
\[
 \nabla f_{j,y}(w)=(w_j-y)e_j,
 \qquad
 \nabla^2f_{j,y}(w)=e_je_j^\top.
\]
On $\cW_d$, $|w_j-y|\le3/2$, so
\[
 \nabla^2f_{j,y}(w)
 \succeq\frac49\nabla f_{j,y}(w)\nabla f_{j,y}(w)^\top,
\]
and the claimed uniform regularity follows.
Let $C_*,c_*$ be the constants in \Cref{lem:effective-lower}, decreasing $c_*$ so that $c_*\le1$. Fix $\kappa=\sqrt{c_*}/2$ and $b=\kappa^2/4$. We may enlarge the universal constant in $T\ge c_0d^2$ so that
\[
 C_*d/\sqrt T\le\min\{c_*T/d^2,bT\},
 \qquad c_*T/d^2\ge(C_*\kappa)^{2/3}.
\]
If $P<C_*d/\sqrt T$, the curved expression is at most $C_*^{2/3}d$, so \Cref{lem:static-lower} suffices. For
$C_*d/\sqrt T\le P\le c_*T/d^2$, use \Cref{lem:effective-lower} with $m=d$.

Suppose next that $c_*T/d^2<P\le bT$ and choose
\[
 m=\left\lfloor\kappa\sqrt{T/P}\right\rfloor.
\]
The upper bound on $P$ ensures that the quantity inside the floor is at least two, so
$\frac\kappa2\sqrt{T/P}\le m\le\kappa\sqrt{T/P}<d$.
In particular, $P\le c_*T/m^2$. The second displayed condition gives
$C_*m/\sqrt T\le C_*\kappa/\sqrt P\le P$.
Thus both conditions of \Cref{lem:effective-lower} hold and its lower bound is $\Omega(\sqrt{TP})$.

At $P=bT$, one of the preceding cases gives $\Omega(T)$: the all-coordinate bound is at least a constant times $T$ if its regime reaches $bT$, and otherwise the effective-dimension bound gives $\sqrt b\,T$. Monotonicity of the admissible comparator class extends this lower bound to every $P\ge bT$. These cases cover every budget, including the constant-factor transition ranges. Finally, the static and moving constructions belong to the same loss class; taking their maximum and using $\max\{a,b\}\ge(a+b)/2$ yields the additive $d$ term.
\end{proof}

\end{document}